\documentclass[conference]{IEEEtran}
\IEEEoverridecommandlockouts
\usepackage{cite}
\usepackage{amsmath,amssymb,amsfonts,bm,amsthm,balance, dsfont}
\usepackage{algorithmic,algorithm}
\usepackage{graphicx}
\usepackage{subcaption}
\usepackage{textcomp}
\usepackage{xcolor,color,colortbl}
\def\BibTeX{{\rm B\kern-.05em{\sc i\kern-.025em b}\kern-.08em
    T\kern-.1667em\lower.7ex\hbox{E}\kern-.125emX}}

 \DeclareFontFamily{OT1}{pzc}{}
\DeclareFontShape{OT1}{pzc}{m}{it}{<-> s * [1.10] pzcmi7t}{}
\DeclareMathAlphabet{\mathpzc}{OT1}{pzc}{m}{it}

\newcommand{\grad}[1]{\nabla_{#1^{\mathsf{T}}} }
\newcommand{\w}{\bm{w}}

\newcommand{\we}{\widetilde{{\w}}}
\newcommand{\eqdef}{\:\overset{\Delta}{=}\:}
\DeclareMathOperator*{\argmin}{argmin}
\newcommand{\Li}[1]{\mathcal{\bm{L}}_{#1,i}}

\newcommand{\tran}{{\sf T}}

\definecolor{Gray}{gray}{0.8}
\definecolor{LightCyan}{rgb}{0.88,1,1}

\newtheorem{theorem}{Theorem}
\newtheorem{assumption}{Assumption}
\newtheorem{lemma}{Lemma}
\newtheorem{definition}{Definition}

\begin{document}

\title{A Graph Federated Architecture with Privacy Preserving Learning}


\author{\IEEEauthorblockN{Elsa Rizk  and Ali H. Sayed\thanks{School of Engineering, École Polytechnique Fédérale de Lausanne, 1015 Lausanne, Switzerland (e-mail:\{elsa.rizk, ali.sayed\}@epfl.ch).}}}

\maketitle

\begin{abstract}
Federated learning involves a central processor that works with multiple agents to find a global model. The process consists of repeatedly exchanging estimates, which results in the diffusion of information pertaining to the local private data. Such a scheme can be inconvenient when dealing with sensitive data, and therefore, there is a need for the privatization of the algorithms. Furthermore, the current architecture of a server connected to multiple clients is highly sensitive to communication failures and computational overloads at the server. Thus in this work, we develop a private multi-server federated learning scheme, which we call graph federated learning. We use cryptographic and differential privacy concepts to privatize the federated learning algorithm that we extend to the graph structure. We study the effect of privatization on the performance of the learning algorithm for general private schemes that can be modeled as additive noise. We show under convexity and Lipschitz conditions, that the privatized process matches the performance of the non-private algorithm, even when we increase the noise variance.  
\end{abstract}


\section{Introduction}
Federated learning (FL) \cite{mcmahan16} has emerged in  recent years as a powerful distributed learning algorithm that aims at finding a global model that fits local data. The FL algorithm consists of two steps: an update step done locally at each client, and an aggregation step done at the server. During these two steps, communication occurs between the clients and the \textit{one} server. Unfortunately, such a structure is not robust, since it relies on one server to carry out all the communications and aggregations. One solution was suggested in \cite{HFL} introducing hierarchical federated learning; the architecture consists of one cloud server connected to a number of edge servers that, in turn, are connected to multiple clients, thus forming a tree structure. However, in this work, we consider a more general framework. We introduce what we call \textit{graph federated learning} (GFL) that consists of several servers each connected to their own subset of clients. The connection between the servers is represented by a graph. Such an architecture is more realistic for instance when considering cellular networks that consist of multiple cellphone towers, each open for communication with numerous cellular devices.

In addition, we focus on the privacy of the federated learning algorithm. It is not sufficient that the local data is not explicitly communicated for the algorithm to be private. The model and gradient updates shared by each client may carry information about the data \cite{Hitaj2017,Melis2019,nasr2019comprehensive,Zhu2019}. For example, if we consider a logistic risk function, the gradient can be expressed as a constant multiplying the feature vector. Therefore, there is a need to privatize the federated algorithm in order to stop the information leakage. 

Multiple solutions exist to privatize distributed learning algorithms. They can be split into two frameworks: differential privacy \cite{geyer2017differentially,hu2020personalized,triastcyn2019federated,truex2020ldp,wei2020federated,JayDLDP,LiDLDP,ZhuDPDL,pathak2010DP}, and cryptography \cite{bonawitz2016practical,gascon2017privacy,Mohassel2017SecureML,Niko2013,Zheng2019}. No framework prevails over the other. While differential privacy is easy to implement, it adds a bias to the solution. On the other hand, cryptographic methods such as secure multi-party computation (SMC) are harder to implement and impose hard limitations on the number of participating parties. Thus, in this work we wish to benefit from the two approaches. We develope a protocol based on the works in \cite{vlaski2020graphhomomorphic} and \cite{bonawitz2016practical}. The former utilizes differential privacy to privatize a distributed learning algorithm on a graph. Unlike standard differential privacy schemes, the perturbations are not independent, but instead, they are choosen to satisfy a nullspae condition determined by the graph structure. While, the latter encorporates multiple SMC tools into the federated learning architecture. However, their scheme can be summarized as adding local perturbations, which will be canceled out at the server, to the updates that have been transformed by an invertible function. 

In this work, we first study the effect of privatization on the performance of the learning algorithm. We study general private algorithms whose privatization schemes can be modelled as added noise, whether it be using differential privacy or SMC. We then present the protocol we adopt and specialize the results.

\section{Problem Set Up}
\begin{figure}
	\begin{center}
		\includegraphics[scale=0.3]{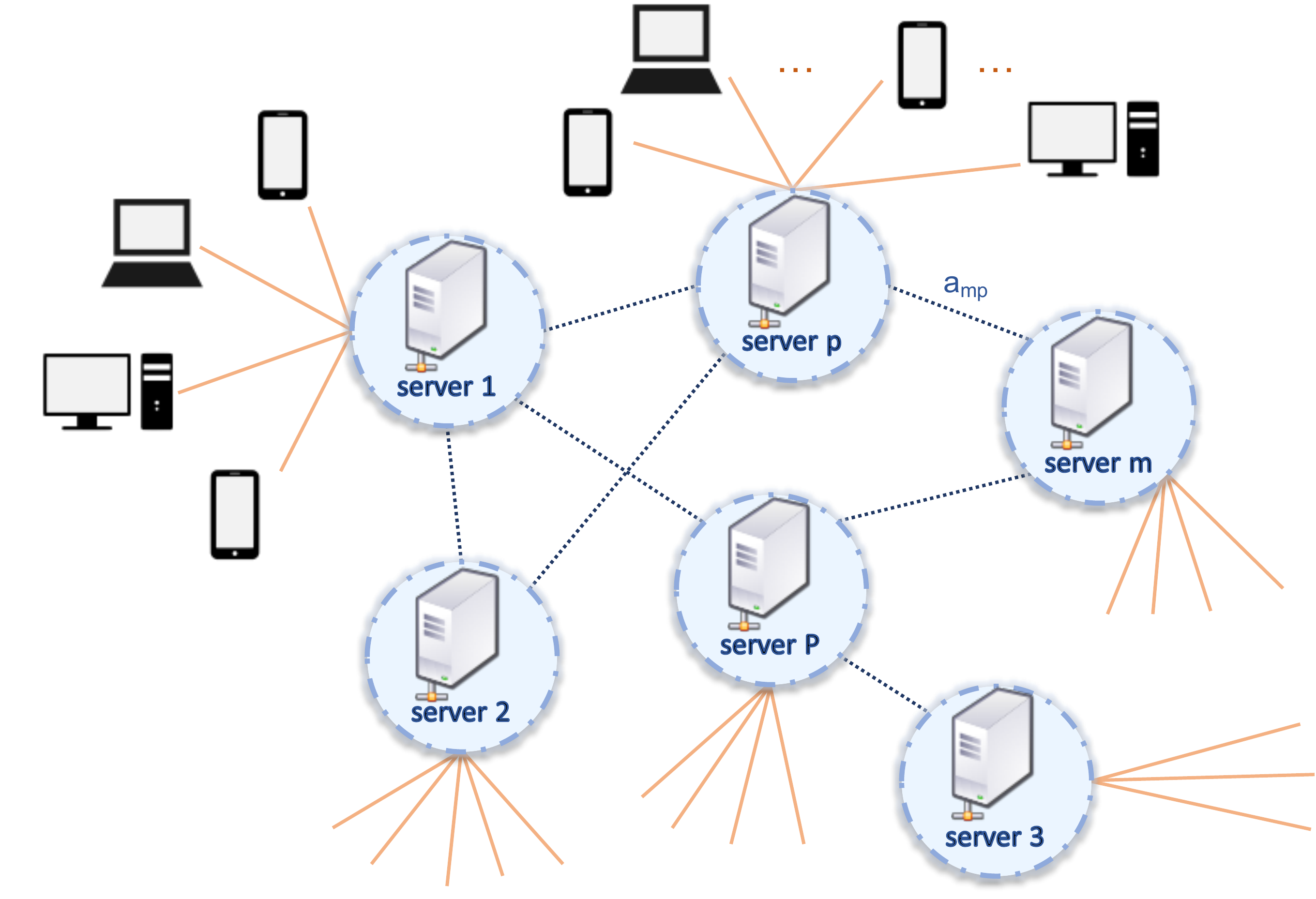}
		\caption{Graph federated architecture.}\label{fig:GFL-arch}
	\end{center}
\end{figure}
The graph federated learning architecture consists of $P$ servers, each connected to a set of $K$ clients, as depicted in Figure \ref{fig:GFL-arch}. The graph connecting the servers is represented by a combination matrix $A \in \mathbb{R}^{P\times P}$ whose elements are denoted by $a_{mp}$. The goal is to minimize the average empirical risk:
\begin{equation}\label{eq:optProb}
	w^{o} \eqdef \argmin_{w\in \mathbb{R}^M} \frac{1}{P}\sum_{p=1}^P \frac{1}{K}\sum_{k=1}^K P_{p,k}(w),
\end{equation} 
where each empirical risk $P_{p,k}(\cdot)$ is defined over a loss function $Q_{p,k}(\cdot;x_{p,k,n})$:
\begin{equation}
	P_{p,k}(w) \eqdef \frac{1}{N_{p,k}} \sum_{n=1}^{N_{p,k}} Q_{p,k}(w;x_{p,k,n}),
\end{equation}
We introduce the subscript $p$ to denote the server, while the subscript $k$ refers to the client and $n$ to the data. To solve problem \eqref{eq:optProb}, each server with its clients runs the federated averaging (FedAvg) algorithm \cite{mcmahan16}, and then the servers amongst themselves run a consensus type algorithm. In our previous works \cite{rizk2020federated} and \cite{rizk2020optimal}, we have shown that when each agent runs a different number of epochs before sending their final update to the server, the resulting incremental error is on the order of $O(\mu^2)$ and is dominated by gradient noise. Thus, to make this work simpler, we assume that the $L$ sampled clients in the set $\Li{p}$ run one stochastic gradient descent (SGD) step during each iteration. More formaly, at iteration $i$, each agent $k \in \Li{p}$ updates the model at the server $\w_{p,i-1}$ to $\w_{p,k,i}$, which they then send to server $p$:
\begin{equation}\label{eq:clientUp}
	\w_{p,k,i} = \w_{p,i-1}  - \mu \frac{1}{B_{p,k}}  \sum_{b\in \mathcal{B}_{p,k,i}} \grad{w}Q_{p,k}(\w_{p,i-1}; \bm{x}_{p,k,b}),
\end{equation}
where $\mathcal{B}_{p,k,i}$ is the mini-batch sampled by client $k$, connected to server $p$, at iteration $i$ and of size $B_{p,k}$. Next, neighbouring servers communicate amongst each other the recieved updates:
\begin{equation}\label{eq:aggr}
	\bm{\psi}_{p,i} = \frac{1}{L} \sum_{k\in \Li{p}} \w_{p,k,i},
\end{equation} 
to finally get:
\begin{equation}\label{eq:serverComb}
	\w_{p,i}  = \sum_{m \in \mathcal{N}_p} a_{mp} \bm{\psi}_{m,i}.
\end{equation}
Next, to introduce privacy to the algorithm, updates sent during each communication round can be perturbed by some noise. Thus, at iteration $i$, let $\bm{g}_{mp,i}$ be the noise added by server $m$ to the update sent to server $p$, and $\bm{g}_{p,k,i}$ be the noise added by agent $k$ to the update sent to server $p$. Then, the algorithm can be described by a client update step \eqref{eq:privClientUp}, a server aggregation step \eqref{eq:privAggr}, and a server combination step \eqref{eq:privComb}.
\begin{align}
	\w_{p,k,i} & = \w_{p,i-1} - \mu \frac{1}{B_{p,k}} \sum_{b\in\mathcal{B}_{p,k,i}} \grad{w}Q_{p,k}(\w_{p,i-1};\bm{x}_{p,k,b}) \label{eq:privClientUp}\\
	\bm{\psi}_{p,i} &= \frac{1}{L}\sum_{k \in \Li{p}} \left( \w_{p,k,i} + \bm{g}_{p,k,i} \right), \label{eq:privAggr} \\
	\w_{p,i} &= \sum_{m \in \mathcal{N}_p} a_{mp} \left( \bm{\psi}_{m,i} + \bm{g}_{mp,i} \right) \label{eq:privComb}
\end{align}
Furthermore, if we assume we are using SMC tools, like secret sharing, we can model the protocol by an invertible function $f(\cdot)$ that maps the local update to an encrypted version. Thus, in the server aggregation \eqref{eq:privAggr} and server combination \eqref{eq:privComb} steps, we replace $\bm{w}_{p,k,i}$ and $\bm{\psi}_{m,i}$ with $f(\bm{w}_{p,k,i})$ and $f(\bm{\psi}_{m,i})$, respectively. For the remainder of the paper, we shall continue with the algorithm formulation in \eqref{eq:privClientUp}-\eqref{eq:privComb} instead of introducing $f(\cdot)$, for ease of notation.
\section{Performance Analysis}
\subsection{Modeling Conditions}
Certain reasonable assumptions on the nature of the graph and the cost functions are made to allow for a tractable convergence analysis. 
\begin{assumption}[Adjacency matrix] \label{ass:adj}
	The adjacency matrix $A$ describing the graph is symmetric and doubly-stochastic, i.e.:
	\begin{equation}
		a_{pm} = a_{mp}, \quad \sum_{m=1}^P a_{mp} = 1.
	\end{equation}
	Furthermore, it is fully connected, satisfying:
	\begin{equation}
		\lambda \eqdef \rho (A - \frac{1}{P}\mathds{1}\mathds{1}^{\tran}) < 1.
	\end{equation}
\qed
\end{assumption}
\begin{assumption}[Convexity and smoothness]\label{ass:conv+smooth}
	The empirical risks $P_{p,k}(\cdot)$ are $\nu-$strongly convex, and the loss functions $Q_{p,k}(\cdot; \cdot)$ are convex, namely:
	\begin{align}
		&P_{p,k}(w_2) \geq P_{p,k}(w_1) + \grad{w}P_{p,k}(w_1)(w_2-w_1) \notag \\
		&\qquad \qquad\quad+ \frac{\nu}{2}\Vert w_2 - w_1\Vert^2,  \label{eq:assStrConv} \\
		&Q_{p,k}(w_2; \cdot) \geq Q_{p,k}(w_1;\cdot) + \grad{w} Q_{p,k}(w_1;\cdot) (w_2-w_1). \label{eq:assConv}
	\end{align} 
	Furthermore, the loss functions have $\delta-$Lipschitz gradients:
	\begin{equation}\label{eq:assLip}
		\Vert \grad{w} Q_{p,k}(w_2;\cdot) - \grad{w} Q_{p,k}(w_1;\cdot) \Vert \leq \delta \Vert w_2 - w_1\Vert.
	\end{equation}
\qed
\end{assumption}
Note that in our previous work \cite{rizk2020federated,rizk2020optimal}, we assumed that the local optimal models, which optimize $P_{p,k}(\cdot)$ at the agents, do not differ too much from the global optimal model at the server. We do not make such an assumption here since we are assuming each agent perfroms one epoch during the agent update step. More explicitly, the bound on the model dissagreement only appears in the incremental error term which we do not have here. If we were to assume that the clients perform mulitple SGD steps in one model update step, then we would need such an assumption to make sure the incremental noise is bounded. However, this assumption is not restrictive, since if the local models differed too much, then collaboration would be nonsensical.

\begin{assumption}[Bounded gradients]\label{ass:bdGrad}
	The norm of the stochastic gradients is bounded along the trajectory of the algorithm:
	\begin{equation}\label{eq:ass-bdGrad}
		\Vert \grad{w}Q_{p,k}(w;\cdot)\Vert \leq B
	\end{equation}
\qed
\end{assumption}
The final condition \eqref{eq:ass-bdGrad} is not an assumption, but is something that can be proved. The bound on the gradient norm is required in the privacy analysis of the algorithm. In general, it is assumed that the gradients are uniformly bounded, and when that does not hold, as in the case of strongly convex cost functions, normalized gradients are used instead. However, we consider the less restrictive condition of bounding the gradients only on the models calculated by the algorithm.

\subsection{Error Recursion}
We focus on the network centroid $\w_{c,i}$ defined by:
\begin{equation}\label{eq:netCentDef}
	\w_{c,i} \eqdef \frac{1}{P}\sum_{p=1}^P \w_{p,i}.
\end{equation}
By combining the three steps of the algorithm, we can get the following recursion for the network centroid:
\begin{align}\label{eq:netCentRec}
	\w_{c,i} 
	=& \w_{c,i-1} - \mu \frac{1}{P}\sum_{p=1}^P \widehat{\grad{w}P_{p}}(\w_{p,i-1}) \notag \\ 
	&+ \frac{1}{P L}\sum_{p=1}^P\sum_{k\in\Li{p}} \bm{g}_{p,k,i} + \frac{1}{P}\sum_{p=1}^P\sum_{m=1}^P a_{mp}\bm{g}_{mp,i},
\end{align}
where we define the stochastic gradient at server $p$ as:
\begin{equation}\label{eq:sgServer}
	\widehat{\grad{w}P_p}(\cdot) \eqdef  \frac{1}{L}\sum_{k \in \Li{p}} \frac{1}{B_{p,k}}\sum_{b\in \mathcal{B}_{p,k,i}}\grad{w}Q_{p,k}(\cdot; \bm{x}_{p,k,b}),
\end{equation}
of the true gradient $\grad{w}P_p(\cdot)$. By defining $\we_{c,i} = w^{o}-\w_{c,i}$ and the gradient noise:
\begin{equation}\label{eq:sgNoise}
	\bm{s}_i = \frac{1}{P}\sum_{p=1}^P \left( \widehat{\grad{w}P_{p}}(\w_{p,i-1}) -  \grad{w}P_p(\w_{p,i-1}) \right),
\end{equation}
  we can write the following error recursion:
\begin{equation}\label{eq:centErrRec}
	\we_{c,i} = \we_{c,i-1} + \mu \frac{1}{P}\sum_{p=1}^P \grad{w}P_p(\w_{p,i-1})  + \mu \bm{s}_i - \bm{g}_{c,i},
\end{equation}
with $\bm{g}_{c,i}$ capturing the total added noise.

\subsection{Generalized Convergence Results}
Before moving to the result on the network convergence, we first introduce the following preliminary lemma. We show that all the models at the servers $\{ \w_{p,i}\}_{p=1}^P$ remain significantly close to the network centroid $\w_{c,i}$.

\begin{lemma}[Network Disagreement] \label{lem:netDis}
	The average deviation from the centroid is bounded during each iteration $i$:
	\begin{align}\label{eq:lemNetDis}
	&	\frac{1}{P} \sum_{p=1}^P \mathbb{E}\Vert \w_{c,i} - \w_{p,i} \Vert^2 \leq 2\mathbb{E}\Vert \we_{c,i-1}\Vert^2 + O(\mu)\sigma_s^2 + O(\sigma_{g_c}^2),
	\end{align}
where $O(\sigma_{g_c}^2)$ is a variance term that depends on the variance of the added noise $\bm{g}_{mp,i}$ and $\bm{g}_{p,k,i}$, and $\sigma_s^2$ is the variance of the gradient noise given by:
\begin{equation}
	\sigma_s^2 \eqdef \frac{2}{P K}\sum_{p=1}^P\sum_{k=1}^K \mathbb{E} \Vert \grad{w}Q_{p,k}(w^o;\bm{x})\Vert^2.
\end{equation}	

\end{lemma}
\begin{proof}
	Proof omitted due to space limitations
\end{proof}
We observe that the added noise contributes an added $O(\sigma_{g_c}^2)$ to the bound, which does not exist in the non-private algorithm's bound. Furthermore, the bound is in terms of the centroid error $\we_{c,i-1}$. As seen in the main theorem below, that term converges to a neighbourhood around zero. 

 \begin{theorem}[Convergence of MSE] \label{thrm:conv}
	Under asumptions \ref{ass:adj} and \ref{ass:conv+smooth}, the network centroid converges to the optimal point $w^{o}$ exponentially fast for a sufficientlt small step size $\mu$ :
	\begin{equation}\label{eq:thrmMSE}
		\mathbb{E}\Vert \we_{c,i}\Vert^2 \leq  O(\mu)\sigma_s^2 + O(\mu + \mu^{-1})\sigma_{g_c}^2 + O(\mu^3).
	\end{equation}
\end{theorem}
\begin{proof}
Proof omitted due to space limitations.	
\end{proof}
Thus, a close examination of the above theorem reveals that all privatized algorithms that can be modelled by added noise, add a noise variance term scaled by $O(\mu + \mu^{-1})$. The $O(\mu^{-1})$ term comes from the noise added at the client level to the updates sent to the server, while the $O(\mu)$ term comes from the network disagreement between the models at the server and the centroid model. The result does not come as a surprise, since it quantifies the trade-off between privacy and accuracy.

\subsection{Performance of the hybrid scheme}
We now specialize the above results to the scheme adopted in this work. The protocol developed in \cite{bonawitz2016practical} utilizes a secret sharing method to insure that the messages sent by the clients arrive to the server encoded. The method is equivalent to applying a mask to the updates by each client, which cancels out at the server, i.e., at every server $p$ the following holds:
\begin{equation}
	\sum_{k \in \Li{p}} \bm{g}_{p,k,i} = 0.
\end{equation}
Furthermore, we apply graph homomorphic perturbations, introduced in \cite{vlaski2020graphhomomorphic}: Let each server $p$ sample independently from the Laplace distribution $\bm{g_{p,i}} \sim Lap(0,\sigma_g/\sqrt{2})$ with variance $\sigma_g^2$. Then, the noise sent among servers can be constructed as:
\begin{equation}
	\bm{g}_{mp,i} = \begin{cases}
		\bm{g}_{m,i}, & \text{if } m \neq p, \\
		- \frac{1-a_{mm}}{a_{mm}}\bm{g}_{m,i}, & \text{if } m = p.
	\end{cases}
\end{equation}
Thus, the following result holds:
\begin{equation}
	\frac{1}{P}\sum_{p=1}^P \sum_{m=1}^P a_{mp}\bm{g}_{mp,i} = 0.
\end{equation}
Therefore, with this scheme, the centroid model recursion \eqref{eq:netCentRec} has no noise component. This implies that the $O(\mu^{-1})$ term disappears from the bound of the MSE \eqref{eq:thrmMSE}. Eventhough the effect of the noise added by the servers remains, it is scaled by $\mu$ in the MSE bound.

\section{Privacy Analysis}
We focus on the privacy of the hybrid scheme described in the previous section. We quantify privacy using differential privacy \cite{dwork2014algorithmic}. Thus, we first need to find the sensitivity of the graph FedAvg algorithm, since it is used to callibarate the perturbations. To do so, consider, without loss of generality, that client 1 connected to server 1 decided not to participate, and instead its data $\bm{x}_{1,1}$ was replaced by some other data $\bm{x}'$ with a different distribution. Then, the algorithm will follow a different trajectory $\w_{p,k,i}'$. The sensitivity of the function is thus given by:
\begin{equation}\label{eq:sensitivity}
		\Delta (i) \eqdef \max_{(p,k)} \Vert \w_{p,k,i} - \w_{p,k,i}'\Vert \leq 2\mu B i
\end{equation}

Next, we wish to show that the algorithm is differentially private, but before doing so we present the definition of $\epsilon(i)-$differential privacy.

\begin{definition}[$\epsilon(i)-$Differential Privacy]
	We say that the algorithm given in \eqref{eq:privClientUp}-\eqref{eq:privComb} is $\epsilon(i)-$differentially private for server $p$ at time $i$ if the following condition holds:
	\begin{equation}\label{eq:def-epsDP}
		\frac{\mathbb{P}\left( \left\{ \left\{ \bm{\psi}_{p,j} + \bm{g}_{mp,j} \right\}_{m \in \mathcal{N}_p\char`\\ \{p\}} \right\}_{j=0}^i\right)}{\mathbb{P}\left( \left\{ \left\{ \bm{\psi^{\prime}}_{p,j} + \bm{g}_{mp,j} \right\}_{m \in \mathcal{N}_p\char`\\ \{p\}} \right\}_{j=0}^i\right)}
		 \leq e^{\epsilon(i)} 		
	 \end{equation}
 \qed
\end{definition}
\noindent
The above definition states that the probability of any trajectory is comparable whether or not a client shares its data. Furthermore, it will be our goal to have a small $\epsilon(i)$ to get a higher privacy guarantee.

\begin{theorem}[Privacy of GFL Algorithm]
	If the algorithm \eqref{eq:privClientUp}-\eqref{eq:privComb} adopts the hybrid privacy scheme described in the previous section, then it is $\epsilon(i)-$differentially private, at time $i$ for a standard deviation of $\sigma_g = \sqrt{2}\mu B(1+i)i/\epsilon(i)$ .
\end{theorem} 
\begin{proof}
	Proof omitted due to space limitations.
\end{proof}
Thus, if we wish to keep the privacy high as more iterations are performed, then the variance of the added noise ought to be increased. This clearly decreases the model utility, as seen in Theorem \ref{thrm:conv}. Another way of interpreting the privacy theorem, as follow: If we keep $\sigma_g$ fixed, then $\epsilon(i) = \sqrt{2}\mu B(1+i)i/\sigma_g = O(i^2)$, which increases as more iterations are performed. Thus, the privacy decreases quadratically with time. This does not come as a surprise, since the longer the algorithm runs, the more information across servers and clients is difussed.

\section{Experimental Results}
To illustrate the theoretical results numericaly, we simulate a GFL consisting of $P = 10 $ servers, each with $K = 50$ clients, whose goal is to solve a logistic regression binary problem. We generate a set of data points $\{\gamma_{p,k}(n), h_{p,k,n}\}_{n=1}^{100}$ for each client, where $\gamma_{p,k}(n) = \pm 1$, and $h_{p,k,n} \in \mathbb{R}^M$ with $f(h_{p,k,n} |\gamma_{p,k}(n) = \gamma) = \mathcal{N}(\gamma;\sigma_{h,p,k}^2)$. We compare our private scheme with a standard private algorithm that uses standard perturbations, and with the non-private algorithm. The results are found in Figure \ref{fig:expRes}. We observe that our hybrid scheme does a very good job at approximating the non-private scheme. We also increase the noise variance, and we observe that while the IID private scheme does not converge, our scheme continues to perform as well as the non-private one.
\begin{figure}
	\begin{center}
		\includegraphics[scale=0.6]{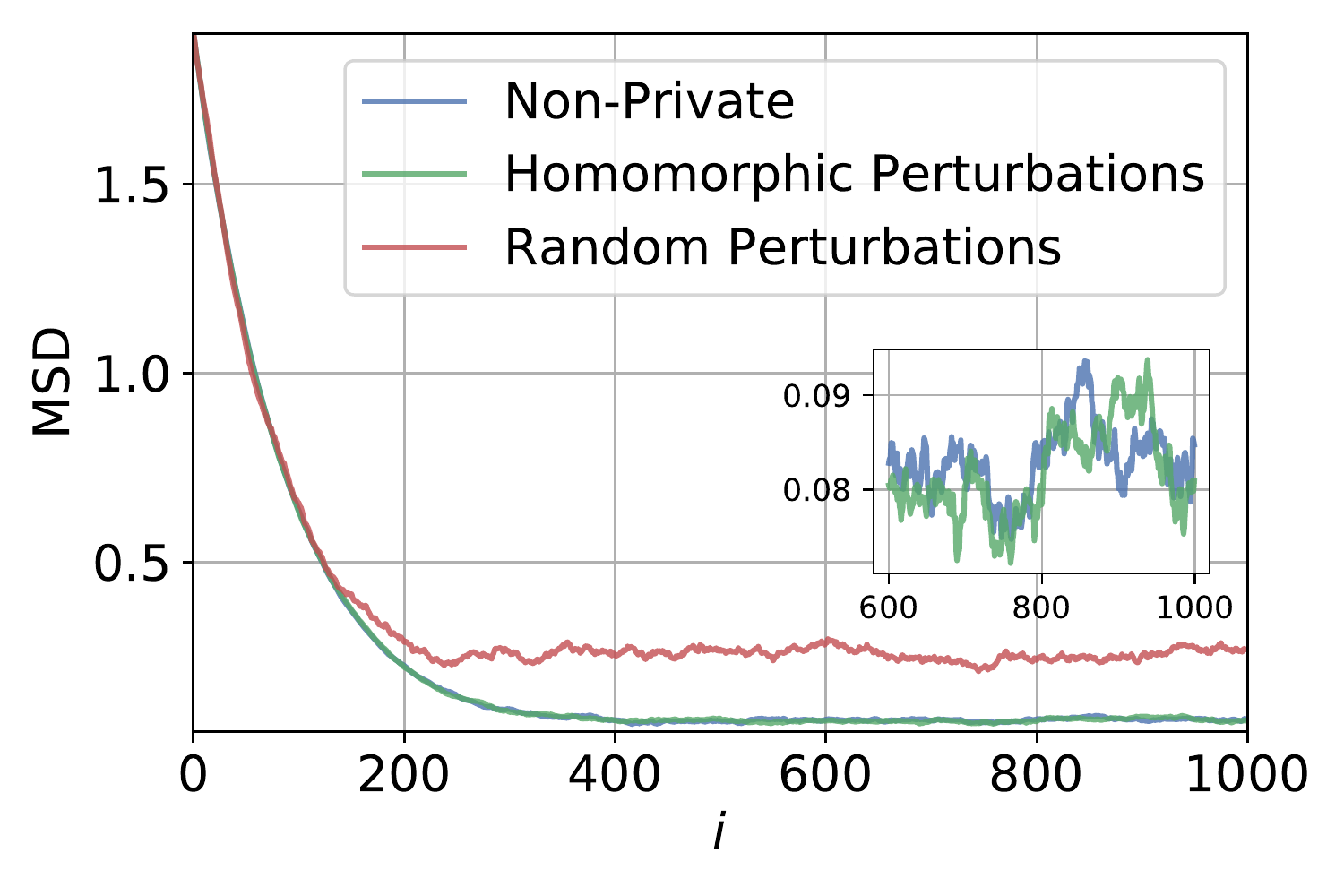}
		\caption{Performance plots with $M = 2$, $\mu=0.1$, $\rho=0.01$, $\sigma_g = 0.2$}\label{fig:expRes}
	\end{center}
\end{figure}

\section{Conclusion}
In this work, we extended the federated learning architecture to GFL. We privatized the algorithm by using non-random perturbations. We quantified the privacy of our algorithm using differential privacy and  provided a performance analysis. Both the theoretical and experimental results showed that non-random perturbations reduce the negative effect of added noise to the model utility. 

\appendices

\bibliographystyle{IEEEtran}
{\balance{\bibliography{refs}}}

\end{document}